\documentclass[12pt]{article}
\usepackage{abstract,amsmath,amssymb,latexsym}
\usepackage{enumitem,epsf}
\usepackage{fullpage,tikz,float}
\usepackage[numbers]{natbib}
\usepackage[pdftex,colorlinks]{hyperref}
\usepackage{algorithm}
\usepackage{algorithmic}

\usepackage{macros}

\newif\ifdraft
\drafttrue

\title{Memorizing Gaussians with no over-parameterizaion via gradient decent on neural networks}

\author{
	\vspace{1cm}
  Amit Daniely\thanks{Hebrew University and Google} 
}

\begin{document}
\maketitle
\setcounter{page}{0}

\thispagestyle{empty}
\maketitle

\begin{abstract}
We prove that a single step of gradient decent over depth two network, with $q$ hidden neurons, starting from orthogonal initialization, can memorize $\Omega\left(\frac{dq}{\log^4(d)}\right)$ independent and randomly labeled Gaussians in $\reals^d$. The result is valid for a large class of activation functions, which includes the absolute value.
\end{abstract}

\newpage

\section{Introduction}

In recent years, much attention has been given to the ability of neural networks, trained with gradient methods, to memorize datasets (e.g. \citep{xie2016diverse, daniely2016toward,  daniely2017sgd, brutzkus2017sgd, oymak2018overparameterized, du2018gradient, allen2018learning, allen2018convergence, cao2019generalization, zou2019improved,  song2019quadratic, ge2019mildly, oymak_towards_2019, arora2019fine,  ziwei2019polylogarithmic, cao2019generalization, ma2019comparative, lee2019wide, daniely2019neural}). The main question is ``how large the networks should be in order to memorize a given dataset $S = \left\{(\x_1,y_1),\ldots,(\x_m,y_m)\right\}\subset \reals^d\times \{\pm,1\}$?"
Here, an example is considered memorized if $y_ih(\x_i) > 0$ for the learned function $h$.

In order to memorize even just slightly more that half of the $m$ examples we need a network with at least $m$ parameters (up to poly-log factors). 
In this paper we will focus on the regime in which the number of parameters is $\tilde{O}(m)$. We will refer to this regime as {\em near optimal memorization}.
To the best of our knowledge, there are very few results that proves near optimal memorization: \citet{brutzkus2017sgd} implies near optimal memorization of linearly independent points (in particular, $m\le d$). \citet{ge2019mildly} implies near optimal memorization of $m\le d^2$  points in general position if the activation is quadratic.  Lastly, \citet{daniely2019neural} shows near optimal memorization of random points in the sphere, for many activation functions, but requires weights initialization that is far from standard, and essentially  makes the optimization process equivalent to NTK optimization \cite{jacot2018neural}.

In this paper we prove near optimal memorization of $m$ ($d$-dimensional) Gaussians, by depth-two network  trained with gradient decent, starting from standard orthogonal initialization, and for a large family of activation functions.

\paragraph{Main Result.} The input examples are denoted $(\x_1,y_1),\ldots,(\x_m,y_m)$. 
We assume that the $\x_i$'s sampled independently from $\cn(0,I_d)$, and the $y_i$'s are independent Rademacher random variables. The initial matrix $W \in M_{q,d}$ is assumed to be orthonormal.
The activation $\sigma:\reals\to\reals$ is assumed to be  (1) $O(1)$ Lipschitz, (2) piecewise twice differentiable with finitely many pieces and a uniform bounded on the second derivative in any piece, and (3) satisfies $\E_{X\sim\cn(0,1)}\sigma'(X)  = 0 $.
An example for such an activation function is the absolute value.

We consider depth two network which calculates the function
\[
h_W(\x) = \frac{1}{\sqrt{q}}\sum_{i=1}^qa_i\sigma(\inner{\bw_i,\x})
\]
Where $a_i\in \{\pm 1\}$ satisfy  $\sum_{i=1}^qa_i = O\left(\sqrt{q}\right)$ (note that this is valid w.h.p. is the $a_i$'s are random).
We consider a single gradient step on $W$, with step size of $\eta = \frac{m\ln(d)}{d}$,
w.r.t. the scaled hinge loss $\ell:\reals\times\{\pm 1\}\to [0,\infty)$, given by
\[
\ell(\hat y,y) = (\ln(d) - \hat y y)_+
\]
We denote by $W^+$ the weights after this single gradient step
\begin{theorem}\label{thm:main}
Assume that $m \le \frac{dq}{\log^4(d)}$ and that $q\ge \log^4(d)$. 
We have that w.p. $1-o(1)$, for every $i\in [m]$, $y_ih_{W^+}(\x_{i}) = \Omega\left( \ln(d) \right)$.
\end{theorem}

\paragraph{Open Questions}
Several obvious open questions arise from our work: To generalize the result to stochastic gradient decent, to more activation functions (and in particular, to the ReLU activation), to non-Gaussian inputs, and to more initialization schemes.

\section{Proof of theorem \ref{thm:main}}

\subsection{Some Tail Inequalities}

Proof of all claims made in this section can be found in chapters 2 and 5 of \citet{vershynin2019high}.
For a reals random variable $X$ and $p\ge 1$ we denote
\[
\|X\|_{\Psi_p} = \inf\{t  : \E\exp\left(|X|^p/t^p\right) \le 2\}
\]
We say that $X$ is $\sigma$-Sub-gaussian if $\|X - \E X\|_{\Psi_2} \le \sigma$. Likewise, we say that $X$ is $\sigma$-Sub-exponential if $\|X - \E X\|_{\Psi_1} \le \sigma$. We will use the following facts.
In the following claims $c$ and $C$ denote positive universal constants.

\begin{lemma}\label{lem:sub_g_and_e}
\begin{enumerate}
\item $\| X - \E X\|_{\Psi_1} \le C \| X \|_{\Psi_1}$ and $\| X - \E X\|_{\Psi_2} \le C \| X \|_{\Psi_2}$ 
\item $\|XY\|_{\Psi_1} \le \|X\|_{\Psi_2}\|Y\|_{\Psi_2}$
\item If $X\sim \cn(0,\sigma)$ then $\|X\|_{\Psi_2} \le C\sigma$
\item $\|X\|_{\Psi_2} \le C\|X\|_\infty$
\item $\Pr\left(|X| \ge t \right) \le 2\exp\left(-ct^2/\|X\|_{\Psi_2}^2\right)$ 
\item $\Pr\left(|X| \ge t \right) \le 2\exp\left(-ct/\|X\|_{\Psi_1}\right)$ 
\end{enumerate}
\end{lemma}

\begin{theorem}[Hoeffding]\label{thm:hoef}
For independent and centered real random variables $X_1,\ldots,X_N$ we have
\[
\left\|\sum_{i=1}^NX_i\right\|^2_{\Psi_2} \le C\sum_{i=1}^N\left\|X_i\right\|^2_{\Psi_2}
\]
In particular, 
\[
\Pr\left(\left|\sum_{i=1}^NX_i\right| \ge t \right) \le 2\exp\left(-\frac{ct^2}{\sum_{i=1}^N\left\|X_i\right\|^2_{\Psi_2}}\right)
\]
\end{theorem}

\begin{theorem}[Bernstein]\label{thm:bern}
For independent and centered real random variables $X_1,\ldots,X_N$ we have
\[
\Pr\left(\left|\sum_{i=1}^NX_i\right| \ge t \right) \le 2\exp\left(-c\min\left(\frac{t^2}{\sum_{i=1}^N\left\|X_i\right\|^2_{\Psi_1}},
\frac{t}{\max_i\left\|X_i\right\|_{\Psi_1}}
\right)\right)
\]
\end{theorem}

\begin{theorem}[Gaussian Concentration]\label{thm:gauss}
Suppose that $X\sim\cn(0,I_n)$ and that $f:\reals^n\to\reals$ is $L$-Lipschitz. Then
\[
\|f(X) - \E f(X) \|_{\Psi_2} \le CL
\]
\end{theorem}

\subsection{Proof}

We first note that 
\begin{lemma}
W.p. $1-o(1)$, for all $i\in [m]$, $|h_W(\x_i)|<O\left(\sqrt{\ln(d)}\right)$.
\end{lemma}
\begin{proof}
If $\x\sim\cn(0,I_d)$ then, since $W$ is orthogonal, $\inner{\bw_1,\x},\ldots,\inner{\bw_q,\x}$ are independent standard Gaussians. Hence, since $\sigma$ is $O(1)$-Lipschitz and by theorem \ref{thm:gauss}, 
$h_W(\x) = \frac{1}{\sqrt{q}}\sum_{i=1}^qa_i\sigma(\inner{\bw_i,\x})$ is a sum of $q$ 
independent $O\left(\frac{1}{\sqrt{q}}\right)$-subgaussians. By theorem \ref{thm:hoef}, $h_W(\x)$  is $O(1)$-subgaussian.
Likewise,  $\E_\x h_W(\x) = \frac{\sum_{i=1}^qa_i}{\sqrt{q}}\E_{X\sim\cn}\sigma(X) = \frac{O(\sqrt{q})}{\sqrt{q}}O(1) = O(1)$. By lemma \ref{lem:sub_g_and_e}, 
for large enough universal constant $C>0$, $\Pr\left(|h_W(\x)|>C\sqrt{\log(d)}\right) < \frac{1}{m^2}$. It follows that $\Pr\left(\exists i\in [m],|h_W(\x_i)|>C\sqrt{\log(d)}\right) < \frac{1}{m}$.
\end{proof}
 
It follows that w.p. $1-o(1)$, for all examples, the hinge loss is in the non-zero part, and we have that $W^+ = W + \eta G$ where
\[
G = \sum_{i=1}^m G^i\text{ for }G^i = \frac{1}{m\sqrt{q}} y_i\diag(\ba)\sigma'(W\x_i)\x_i^T
\]
It is therefore enough to prove the following lemma:
\begin{lemma}\label{lem:main}
Assume that $m \le \frac{dq}{\log^4(d)}$ and that $q\ge \log^4(d)$. We have that w.p. $1-o(1)$, for every $i\in [m]$, $y_ih_{W+ \eta G}(\x_{i}) = \Omega\left( \ln(d) \right)$.
\end{lemma}
In the sequel we denote
\[
\tilde G = \sum_{i=1}^{m-1}G^i
\]
Likewise, we denote by $\tilde \g_j, \g_j$ and $\g^i_j$ the $j$'th row of $\tilde{G}, G$ and $G^i$.
\begin{fact}(e.g. chapter 5 in \cite{van2014probability})\label{fact:eps_net}
There are subsets $S_{d,\epsilon}\subset \sphere^{d-1}$ of size 
$\left(\frac{1}{\epsilon}\right)^{\Theta(d)}$ such that for every matrix $W\in M_{q\times d}$ we have
\[
\|W\| \le (1+\epsilon)\max_{\bu\in S_{q,\epsilon}, \z\in  S_{d,\epsilon}} \inner{\bu, W,\z}
\] 
\end{fact}

\begin{lemma}\label{lem:G_norm}
We have that $\|\eta \tilde G\| \le 2$ w.p. $\exp\left(O(d)-\Omega\left(\frac{d^2q}{m\ln^2(d)}\right)\right)$
\end{lemma}
\begin{proof}
Let $S_{q,1},S_{d,1}$ be the sets from fact \ref{fact:eps_net} 
We have
\begin{eqnarray*}
\|\eta \tilde G\| &\le & 2\max_{\bu\in S_{q,1},\z\in S_{d,1}}\inner{\bu,\eta \tilde G\z} 
\\
&=&  2\max_{\bu\in S_{q,1},\z\in S_{d,1}}\frac{\ln(d)}{d\sqrt{q}}\sum_{i=1}^{m-1} y_i\inner{\diag(\ba)\sigma'(W\x_i),\bu}\inner{\x_i,\z}
\\
&=&  2\max_{\bu\in S_{q,1},\z\in S_{d,1}}\frac{\ln(d)}{d\sqrt{q}}\sum_{i=1}^{m-1} y_i\inner{\sigma'(W\x_i),\bu}\inner{\x_i,\z}
\end{eqnarray*}
Fix $\bu\in S_{q,1}$ and $\z\in S_{d,1}$. We claim that
\begin{claim}
$\sum_{i=1}^m y_i\inner{\sigma'(W\x_i),\bu}\inner{\x_i,\z}$ is a sum of ${m-1}$ independent and centered $O(1)$-Sub-exponential random variables
\end{claim}
\begin{proof}
Clearly, $\sum_{i=1}^{m-1} y_i\inner{\sigma'(W\x_i),\bu}\inner{\x_i,\z}$ is a sum of ${m-1}$ independent and centered random variables. It remains to prove $O(1)$-Sub-exponentiality. By lemma \ref{lem:sub_g_and_e} it is enough to show that $\inner{\sigma'(W\x_i),\bu}$ and $\inner{\x_i,\z}$ are $O(1)$-Sub-gaussian. 
Indeed, $\inner{\x_i,\z}\sim\cn(0,1)$ and hence by lemma \ref{lem:sub_g_and_e} it is $O(1)$-Sub-gaussian. As for $\inner{\sigma'(W\x_i),\bu} = \sum_{j=1}^q u_j \sigma'(\inner{\bw_j,\x_i})$, we have that $\inner{\bw_1,\x_i},\ldots,\inner{\bw_q,\x_i}$ are independent since $\x_i$ is Gaussian and $W$ is orthogonal. Hence, $\inner{\sigma'(W\x_i),\bu}$ is a sum of independent random variables. Furthermore, for every $j$, $\inner{\bw_j,\x_i}\sim\cn(0,1)$, and since we assume that $\E_{X\sim\cn(0,1)}\sigma'(X)=0$, we conclude that $\inner{\sigma'(W\x_i),\bu}$ is a sum of independent and centered random variables.
We can now use lemma \ref{lem:sub_g_and_e} and theorem \ref{thm:hoef} to conclude that
\begin{eqnarray*}
\left\| \inner{\sigma'(W\x_i),\bu}\right\|_{\Psi_2}^2 &=& \left\| \sum_{j=1}^q u_j \sigma'(\inner{\bw_j,\x_i}) \right\|_{\Psi_2}^2 
\\
&\le& C \sum_{j=1}^q u^2_j \left\|\sigma'(\inner{\bw_j,\x_i}) \right\|_{\Psi_2}^2
\\
&\le& C \| \sigma'\|_\infty \sum_{j=1}^q u^2_j 
\\
&=& C \| \sigma'\|_\infty  
\end{eqnarray*}
\end{proof}
We can now use Bernstein inequality to conclude that 
\[
\Pr\left( \left|\frac{\ln(d)}{d\sqrt{q}} \sum_{i=1}^{m-1} y_i\inner{\sigma'(W\x_i),\bu}\inner{\x_i,\z} \right| \ge t \right) \le \exp\left(-\Omega\left( \min \left( \frac{t^2d^2q}{m\ln^2(d)},\frac{td\sqrt{q}}{\ln(d)} \right)\right)\right)
\]
For $t=\frac{1}{2}$ and $m\ge \frac{d\sqrt{q}}{\ln(d)}$ we get
\[
\Pr\left( \left|\frac{\ln(d)}{d\sqrt{q}} \sum_{i=1}^{m-1} y_i\inner{\sigma'(W\x_i),\bu}\inner{\x_i,\z} \right| \ge \frac{1}{2} \right) \le \exp\left(-\Omega\left(  \frac{d^2q}{m\ln^2(d)} \right)\right)
\]
Via a union bound on $S_{q,1}\times S_{d,1}$ we get that
\[
2\max_{\bu\in S_{q,1},\z\in S_{d,1}}\left|\frac{\ln(d)}{d\sqrt{q}} \sum_{i=1}^{m-1} y_i\inner{\sigma'(W\x_i),\bu}\inner{\x_i,\z} \right| \le 1
\]
w.p. $\exp\left(O(d+q)-\Omega\left(\frac{d^2q}{m\ln^2(d)}\right)\right) = \exp\left(O(d)-\Omega\left(\frac{d^2q}{m\ln^2(d)}\right)\right)$. Finally, the case $m<d\sqrt{q}$ can be reduced to the case $m\ge d\sqrt{q}$ by adding $(d\sqrt{q}-m)$  random variables which are identically $0$, and noting that we are still left with a sum of independent and centered $O(1)$-subexponential random variables.

%For $t=1$ and $\sqrt{q}\le m\le d\sqrt{q}$ we get
%\[
%\Pr\left( \left|\frac{1}{d\sqrt{q}} \sum_{i=1}^m y_i\inner{\sigma'(W\x_i),\bu}\inner{\x_i^T,\z} \right| \ge 1 \right) \le \exp\left(-\Omega\left( d\sqrt{q} \right)\right)
%\]
%By standard $\epsilon$-net argument \todo{details} we conclude that
%\[
%\max_{\bu,\z\in\sphere^{d-1}}\left|\frac{1}{d\sqrt{q}} \sum_{i=1}^m y_i\inner{\sigma'(W\x_i),\bu}\inner{\x_i^T,\z} \right| \le O(1)
%\]
%w.p. $\exp\left(-\Omega\left(d\sqrt{q}\right)\right) \le \exp\left(-\Omega\left(\frac{dq}{m}\right)\right)$. Finally, for $m\le \sqrt{q}$ we have that
%\[
%\left|\frac{1}{d\sqrt{q}} \sum_{i=1}^m y_i\inner{\sigma'(W\x_i),\bu}\inner{\x_i^T,\z} \right| \le O(1)
%\]
  
\end{proof}

\begin{lemma}\label{lem:g_norm}
Assume that $m\le dq$.
For every $i$ we have that 
\begin{enumerate}
\item $1\le \E\|\bw_i +\eta \tilde \g_i\|^2 \le  1 + \frac{m\ln^2(d)}{dq}\|\sigma'\|_\infty^2$. Furthermore, the probability that $\|\bw_i +\eta \tilde \g_i\|^2$ $\frac{1}{\sqrt{d}}$-deviates from its expectation is at most $\exp\left(-\Omega\left(\frac{d^2q^2}{m^2\ln^4(d)}\right)\right)$
\item $ \E\|\eta \tilde \g_i\|^2 \le   \frac{m\ln^2(d)}{dq}\|\sigma'\|_\infty^2$. Furthermore, the probability that $\|\eta \tilde \g_i\|^2$ $\frac{1}{\sqrt{d}}$-deviates from its expectation is at most $\exp\left(-\Omega\left(\frac{d^2q^2}{m^2\ln^4(d)}\right)\right)$
\end{enumerate}
\end{lemma}
\begin{proof}
We will prove the first part of the lemma. The proof of second part is very similar.
Denote $\g = \tilde\g_i$ and $\bw = \bw_i$. Since the input distribution is invariant to orthogonal transformations, we can assume w.l.o.g. we assume that $\bw = \e_1$. We also assume that $a_i = 1$. The case $a_i = -1$ is similar. 
We have
\begin{eqnarray*}
\|\bw +\eta  \g\|^2 &=& \left(\bw(1) + \frac{\ln(d)}{d\sqrt{q}}\sum_{i=1}^{m-1} y_i\sigma'(\x_i(1))\x_i(1)\right)^2 + \sum_{j=2}^d\left(\bw(j) +\frac{\ln(d)}{d\sqrt{q}}\sum_{i=1}^{m-1} y_i\sigma'(\x_i(1))\x_i(j)\right)^2
\\
&=& \left(1 + \frac{\ln(d)}{d\sqrt{q}}\sum_{i=1}^{m-1} y_i\sigma'(\x_i(1))\x_i(1)\right)^2
\\
&& + \sum_{j=2}^d\left[\left(\frac{\ln(d)}{d\sqrt{q}}\sum_{i=1}^{m-1} y_i\sigma'(\x_i(1))\x_i(j)\right)^2  - \frac{\ln^2(d)}{d^2q}\sum_{i=1}^{m-1} \left(\sigma'(\x_i(1))\right)^2\right]
\\
&&  + \frac{\ln^2(d)(d-1)}{d^2q}\sum_{i=1}^{m-1} \left(\sigma'(\x_i(1))\right)^2
\end{eqnarray*} 
Now, by theorem \ref{thm:hoef} we have that $ \frac{\ln(d)}{d\sqrt{q}}\sum_{i=1}^{m-1} y_i\sigma'(\x_i(1))\x_i(1)$ is $O\left(\frac{\sqrt{m}\ln(d)}{d\sqrt{q}}\right)$-sub-gaussian. This implies that the probability that $\left(1 + \frac{\ln(d)}{d\sqrt{q}}\sum_{i=1}^{m-1} y_i\sigma'(\x_i(1))\x_i(1)\right)^2$  $\epsilon$-deviates from its expectation is bounded by $\exp\left(-\Omega\left(\frac{d^2q\epsilon^2}{m\ln^2(d)}\right)\right) + \exp\left(-\Omega\left(\frac{d^2q\epsilon}{m\ln^2(d)}\right)\right) = \exp\left(-\Omega\left(\frac{d^2q\epsilon^2}{m\ln^2(d)}\right)\right)$. Likewise,
\begin{eqnarray*}
\E\left(1 + \frac{\ln(d)}{d\sqrt{q}}\sum_{i=1}^{m-1} y_i\sigma'(\x_i(1))\x_i(1)\right)^2 &=& 1 + \frac{\ln^2(d)}{d^2q}\sum_{i=1}^{m-1}\E\left(\sigma'(\x_i(1))\x_i(1)\right)^2
\\
&\le& 1 + \frac{\ln^2(d)\|\sigma'\|_\infty^2 }{d^2q}\sum_{i=1}^{m-1}\E\left(\x_i(1)\right)^2
\\
&=& 1 + \frac{\ln^2(d){(m-1)}\|\sigma'\|_\infty^2 }{d^2q}
\end{eqnarray*} 

Theorem \ref{thm:hoef} also implies that the last line is $O\left(\frac{\ln^2(d)\sqrt{m}}{dq}\right)$-sub-gaussian. 
Thus, the probability that it $\epsilon$-deviates from its expectation is bounded by $\exp\left(-\Omega\left(\frac{d^2q^2\epsilon^2}{m\ln^4(d)}\right)\right)$. Likewise, its expectation is bounded by $ \frac{\ln^2(d)(d-1)m \|\sigma'\|_\infty^2  }{d^2q}$ from above and by $0$ from below.

Finally, given $x_1(1),\ldots,x_{m-1}(1)$ and $y_1,\ldots,y_{m-1}$, the middle line is a sum of $d-1$ independent random variables. Each of which has zero mean and is $O\left(\frac{m\ln^2(d)}{d^2q}\right)$-sub-exponential. By Berstein inequality, the probability that it $\epsilon$-deviates from its expectation is bounded by $\exp\left(-\Omega\left(\frac{d^3q^2\epsilon^2}{m^2\ln^4(d)}\right)\right) + \exp\left(-\Omega\left(\frac{d^2q\epsilon}{m\ln^2(d)}\right)\right)$.
Choosing $\epsilon = \frac{1}{3\sqrt{d}}$, we conclude that the probability that $\|\bw +\eta \bu\|^2$ $\frac{1}{\sqrt{d}}$-deviates from its expectation is at most $\exp\left(-\Omega\left(\frac{d^2q^2}{m^2\ln^4(d)}\right)\right)$. As for the expectation, since the expectation of the middle line is $0$, the total expectation is bounded by
\[
1 + \frac{\ln^2(d)m\|\sigma'\|_\infty^2 }{d^2q} + 0 +  \frac{\ln^2(d)(d-1)m \|\sigma'\|_\infty^2  }{d^2q} = 1 + \frac{m\ln^2(d)}{dq}\|\sigma'\|_\infty^2
\]
from above and by $1 + 0 + 0 =1$ from below.
\end{proof}

We are now ready to prove lemma \ref{lem:main}, and therefore also theorem \ref{thm:main}.

\begin{proof}(of lemma \ref{lem:main})
We will prove the theorem under the assumption that $\sigma$ is twice differentiable everywhere. We will later expalin how to amend the proof in the case that it is only piece-wise twice differentiable.
It is enough to show that w.p. $1-o\left(\frac{1}{m}\right)$, $y_mh_{W+ \eta G}(\x_{m}) = \Omega\left( \ln(d) \right)$.
Throughout the proof, w.h.p., means "w.p. $1-o\left(\frac{1}{m}\right)$".
Note that if $O(1)$ events holds w.h.p., then so is their union.
We have 
\[h_{W+ \eta  G }(\x_{m}) = h_{W+ \eta \tilde G + \eta G_m}(\x_{m}) - h_{W+ \eta \tilde G }(\x_{m}) + h_{W+ \eta \tilde G }(\x_{m})
\]
The proof of the lemma follows from the following two claims.
\begin{claim}
W.h.p.  $h_{W+ \eta \tilde G }(\x_{m}) = O\left(\sqrt{\log\left(d\right)}\right)$. 
\end{claim}
\begin{proof}
By lemma \ref{lem:G_norm}, and since $m \le \frac{dq}{\log^4(d)}$, we have that w.h.p. $\|W + \eta \tilde G\| \le 3 $. Likewise, lemma \ref{lem:g_norm} implies that w.h.p., for all $i$, $\|\bw_i + \eta \tilde \g_i\| = \sqrt{\E\|\bw_i + \eta \tilde \g_i\|^2} + \delta_i $, where $\delta_i = O\left(\frac{1}{\sqrt{d}}\right)$. We will show that the claim holds w.h.p. given these two events.

First, since $\|W + \eta \tilde G\| \le 3 $, $\x\mapsto h_{W+ \eta \tilde G }(\x)$ is $O(1)$-Lipschitz, as a composition of the $O(1)$-Lipschitz functions $\x\mapsto\left(W+ \eta \tilde G\right)\x$, $\x\mapsto\sigma(\x)$, and $\x\mapsto \frac{1}{\sqrt{q}}\sum_{i=1}^qa_i\x_i$.

It follows that, w.h.p., by Lipschitz Gaussian concentration (theorem \ref{thm:gauss}) we have that $h_{W+ \eta \tilde G }(\x_{m})$, is $O(1)$ Sub-Gaussian. Hence, w.h.p., its distance from its expectation is $O\left(\sqrt{\log\left(d\right)}\right)$. It therefore enough to show that $\E_{\x_m}h_{W+ \eta \tilde G }(\x_{m}) = O(1)$.
Since $\|\bw_i + \eta \tilde \g_i\| = \sqrt{\E\|\bw_i + \eta \tilde \g_i\|^2} + \delta_i $, where $\delta_i = O\left(\frac{1}{\sqrt{d}}\right)$, we can write $\inner{\bw_i+ \eta \tilde \g_i , \x_{m}} = X + Y_i$, where $X$ is a centered Gaussian of variance $\E\|\bw_i + \eta \tilde \g_i\|^2$, and $Y_i$ is a centered Gaussian of variance $O\left(\frac{1}{d}\right)$.
We have that 
\begin{eqnarray*}
\E_{\x_m}h_{W+ \eta \tilde G }(\x_{m}) &=& \frac{1}{\sqrt{q}}\sum_{j=1}^qa_i\E_{\x_m}\sigma\left( \inner{\bw_i+ \eta \tilde \g_i , \x_{m}}\right)
\\
&=& \frac{1}{\sqrt{q}}\sum_{j=1}^qa_i\E_{X,Y_i}\sigma\left( X +  Y_i\right)
\\
&=& \frac{1}{\sqrt{q}}\sum_{j=1}^qa_i\E_{X,Y_i}\sigma\left( X \right) + \sigma\left( X +  Y_i\right) - \sigma\left( X \right) 
\\
&\stackrel{\sum_{j=1}^qa_i = O(\sqrt{q})}{=}& O(1) + \frac{1}{\sqrt{q}}\sum_{j=1}^qa_i\E_{X,Y_i}  \sigma\left( X +  Y_i\right) - \sigma\left( X \right) 
\end{eqnarray*}
Now, for every fixed $x$ we have, since $\sigma$ is $O(1)$-Lipschitz,
\[
\left| \E_{Y_i}  \sigma\left( x +  Y_i\right) - \sigma\left( x \right)  \right| \le O(1)\E_{Y_i}|Y_i| = O\left(\frac{1}{\sqrt{d}}\right) 
\]
It therefore follows that $\E_{\x_m}h_{W+ \eta \tilde G }(\x_{m}) = O(1) + O\left(\sqrt{\frac{q}{d}}\right) = O(1)$
\end{proof}

\begin{claim} W.h.p. $y_{m}\left[ h_{W+ \eta \tilde G + \eta G_m}(\x_{m}) - h_{W+ \eta \tilde G }(\x_{m}) \right] = \Omega\left(\log\left(d\right)\right)$
\end{claim}
\begin{proof}
We first note that by lemma \ref{lem:g_norm} we have that, w.h.p., for every $i\in [q]$, $\|\eta\tilde{g}_i\| = O\left(  \frac{1}{\log(d)}\right)$. Hence, w.h.p, for every $i\in [q]$,  $|\inner{\eta\tilde{g}_i,\x_m}| = O\left(\frac{1}{\sqrt{\log(d)}}\right)$
Fix $i\in [q]$. Recall that  $\eta\g^m_i = \frac{a_iy_m\log(d)}{\sqrt{q}d}\sigma'\left(\inner{\bw_i,\x_m}\right)\x_m$. Likewise, w.h.p., $\frac{\|x_m\|^2}{d} = 1 + o(1)$.
We have that, w.h.p., 
\scriptsize
\begin{eqnarray*}
\sigma\left( \inner{\bw_i+ \eta \tilde \g_i + \eta \g^m_i,\x_m}\right) - \sigma\left( \inner{\bw_i+ \eta \tilde \g_i ,\x_m}\right) &=& \sigma'\left(\inner{\bw_i+ \eta \tilde \g_i ,\x_m}\right)\inner{\eta \g^m_i,\x_m} + O(1)\left(\inner{\eta \g^m_i,\x_m}\right)^2
\\
&=& \frac{\ln(d)}{d\sqrt{q}}\sigma'\left(\inner{\bw_i+ \eta \tilde \g_i ,\x_m}\right)\inner{y_ma_i \sigma'\left(\inner{\bw_i ,\x_m}\right)\x_m,\x_m} 
\\
&& + O(1)\left(\frac{\ln(d)\inner{\x_m,\x_m}}{d\sqrt{q}}\right)^2
\\
&=& \frac{y_ma_i\ln(d)(1 + o(1))}{\sqrt{q}}\sigma'\left(\inner{\bw_i+ \eta \tilde \g_i ,\x_m}\right)\sigma'\left(\inner{\bw_i ,\x_m}\right) 
\\
&& + O\left(\frac{\ln^2(d)}{q}\right)
\\
&\stackrel{\sigma'\text{ is }O(1)\text{-Lip.}}{=}& \frac{y_ma_i\ln(d)(1 + o(1))}{\sqrt{q}}\left(\sigma'\left(\inner{\bw_i ,\x_m}\right) + O\left(|\inner{ \eta \tilde \g_i ,\x_m}|\right)\right)\sigma'\left(\inner{\bw_i ,\x_m}\right) 
\\
&& + O\left(\frac{\ln^2(d)}{q}\right)
\\
&=& \frac{y_ma_i\ln(d)(1 + o(1))}{\sqrt{q}}\left(\sigma'\left(\inner{\bw_i ,\x_m}\right) + o\left(1\right)\right)\sigma'\left(\inner{\bw_i ,\x_m}\right) 
\\
&& + O\left(\frac{\ln^2(d)}{q}\right)
\\
&=& \frac{y_ma_i\ln(d)(1 + o(1))}{\sqrt{q}}\left(\sigma'\left(\inner{\bw_i ,\x_m}\right) \right)^2 + \frac{o(\log(d))}{\sqrt{q}} + O\left(\frac{\ln^2(d)}{q}\right)
\\
&\stackrel{q\ge \log^4(d)}{=}& \frac{y_ma_i\ln(d)(1 + o(1))}{\sqrt{q}}\left(\sigma'\left(\inner{\bw_i ,\x_m}\right) \right)^2 + \frac{o(\log(d))}{\sqrt{q}} 
\end{eqnarray*}
\normalsize
It follows that
\begin{eqnarray*}
 h_{W+ \eta \tilde G + \eta G_m}(\x_{m}) - h_{W+ \eta \tilde G }(\x_{m}) &=& \frac{1}{\sqrt{q}}\sum_{i=1}^qa_i\left(\sigma\left( \inner{\bw_i+ \eta \tilde \g_i + \eta \g^m_i,\x_m}\right) - \sigma\left( \inner{\bw_i+ \eta \tilde \g_i ,\x_m}\right)\right)
 \\
 &=& o(\log(d)) + \frac{y_m\ln(d)(1+o(1))}{q}\sum_{i=1}^q \left(\sigma'\left(\inner{\bw_i ,\x_m}\right) \right)^2
\\
&=& o(\log(d)) + y_m\ln(d)(1+o(1))\E_{X\sim\cn(0,1)}\left(\sigma'(X)\right)^2
\\
&=&  y_m\ln(d)(1+o(1))\E_{X\sim\cn(0,1)}\left(\sigma'(X)\right)^2
\end{eqnarray*}

\end{proof}

To handle the case that $\sigma$ is only piece-wise twice differentiable (with finitely many pieces), one should observe that $1-o(1)$ of the neurons we have that  $\inner{\bw_i,\x_m}$ is well inside one of the pieces, so that the estimation of $\sigma\left( \inner{\bw_i+ \eta \tilde \g_i + \eta \g^m_i,\x_m}\right) - \sigma\left( \inner{\bw_i+ \eta \tilde \g_i ,\x_m}\right)$ is still valid. Likewise, the remaining neurons effect $h_{W+ \eta \tilde G + \eta G_m}(\x_{m}) - h_{W+ \eta \tilde G }(\x_{m}) $ by $o(\log(d))$, and hence the estimation of $h_{W+ \eta \tilde G + \eta G_m}(\x_{m}) - h_{W+ \eta \tilde G }(\x_{m}) $ remains valid.
\end{proof}

\subsubsection*{Acknowledgments}
This research is partially supported by ISF grant 2258/19

\bibliographystyle{plainnat}

\bibliography{bib,bib2}

\end{document}